\documentclass[10pt,journal,compsoc]{IEEEtran}

\newcommand{\eqdef}{\overset{\Delta}{=}}
\usepackage{amsthm}
\usepackage{amssymb}
\newtheorem{lemma}{Lemma}
\theoremstyle{definition}  
\newtheorem*{example}{Example}
\usepackage{color}
\usepackage{xcolor}
\newtheorem{definition}{Definition}
\newtheorem{theorem}{Theorem}

\ifCLASSOPTIONcompsoc
  \usepackage[nocompress]{cite}
\else
  \usepackage{cite}
\fi

\ifCLASSINFOpdf
  \usepackage[pdftex]{graphicx}
\fi
\usepackage{amsmath}
\usepackage{array}

\begin{document}
\title{Unique Geometry and Texture from Corresponding Image Patches}

\author{Dor~Verbin,~\IEEEmembership{Student Member,~IEEE,}
        Steven~J.~Gortler,
        and~Todd~Zickler,~\IEEEmembership{Member,~IEEE}
\IEEEcompsocitemizethanks{\IEEEcompsocthanksitem Authors are with the School
of Engineering and Applied Sciences, Harvard University, Cambridge,
MA, 02138.\protect\\
E-mail: \{dorverbin,sjg,zickler\}@seas.harvard.edu}
}
\IEEEtitleabstractindextext{%
\begin{abstract}
We present a sufficient condition for recovering unique texture and viewpoints from unknown orthographic projections of a flat texture process. We show that four observations are sufficient in general, and we characterize the ambiguous cases. The results are applicable to shape from texture and texture-based structure from motion.
\end{abstract}

\begin{IEEEkeywords}
Shape from texture, structure from motion.
\end{IEEEkeywords}}

\maketitle

\IEEEdisplaynontitleabstractindextext
\IEEEpeerreviewmaketitle

\IEEEraisesectionheading{\section{Introduction}\label{sec:introduction}}

\IEEEPARstart{S}{uppose} we are given a collection of image patches that are the orthographic projections, from various directions,
of a single texture element or stochastic texture process. The collection of image patches may be explained by the texture and the viewpoints that generated it, but it may also be explained by a texture that is a spatially-sheared version of the veridical one, along with a collection of viewing geometries that are distorted. We want to understand when the geometry and texture can be recovered correctly and when they cannot.

This question arises in certain formulations of shape from texture~\cite{lobay2006shape,hartley,verbin}, where the local foreshortening of a spatially-repetitive texture process on a curved surface induces a perception of three-dimensional shape. In this version of the problem, the unknown per-patch geometries are interpreted as the local surface normal and tangent frames, and our quest is to understand when this ``shape'' and accompanying flat-texture process (e.g.,~right column of Fig.~\ref{fig:teaser}) can be correctly identified.

The same question is relevant to certain formulations of structure from motion, when an affine covariant region detector~\cite{matas,mikolajczyk} identifies corresponding ``interest points'' between images captured from distinct, orthographic viewpoints. In this context, we pursue a characterization of conditions that allow for recovering information about the view directions, even when only a single ``interest point''---as in the top row of Fig.~\ref{fig:teaser}---is shared between views.

It has been previously claimed that three distinct directions are sufficient to recover the correct geometry and flat-texture in general~\cite{lobay2006shape}. In this paper we show that in general the actual minimum number of directions is four. We also characterize the ambiguities that can arise when the patches are fewer in number or are geometrically degenerate.

\begin{figure}
    \centering
    \includegraphics[width=0.485\textwidth]{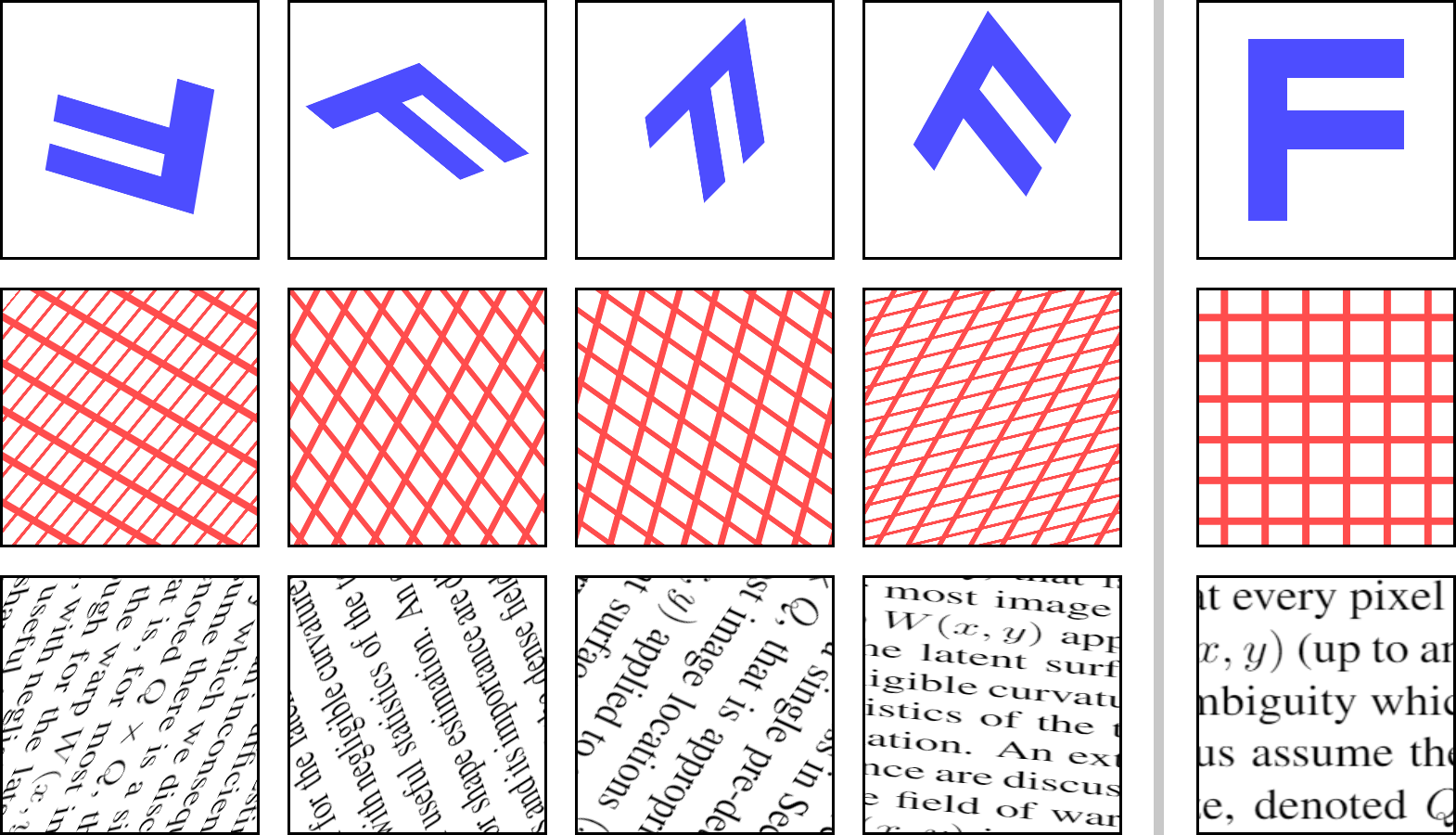}
    \caption{The left of each row is a collection of four image patches generated by orthographic projections of a single flat texture process, shown right. We  want to understand  the conditions that are sufficient for the correct viewing geometries and flat texture process to be recovered from the image patches, when neither the flat texture process nor the viewing geometries are known or labeled beforehand.}
    \label{fig:teaser}
\end{figure}

These results apply to input patches that have already been identified and cropped from images, and that have already been aligned up to translation-free affine transformations, as in the left of Fig.~\ref{fig:teaser}. We do not consider the practical aspects of how to obtain such patches from images in this paper, nor do we consider the sensitivity of the uniqueness results with respect to errors in cropping or alignment. We comment further on the relation to practice in  Section~\ref{sec:limitations}.

Before proceeding, we note that the viewing geometry associated with each image patch has three degrees of freedom: two for the view direction relative to the flat texture's surface normal, and one for the tangent orientation within the texture plane. Given a set of corresponding image patches as in the left of Fig.~\ref{fig:teaser}, one can at best expect to recover the generating tangent orientations relative to an arbitrary coordinate system in the tangent plane. This implies recovering the flat texture process up to a rotation of its two spatial dimensions (e.g., a 2D rotation of the textures on the right of Fig.~\ref{fig:teaser}). In shape from texture, this 2D orthogonal transform does not affect the surface normal so is inconsequential. For our purposes, we consider any geometry/texture explanation that differs from the veridical one by such a relation to be correct. In structure from motion, this means that when only one interest point is shared between viewpoints, we can recover the view directions up to a global rotation about the surface normal.

\section{Texture Cyclostationarity} \label{sec:cyclostationary}

Texture is the spatial repetition of appearance, and the statistical notion of cyclostationarity provides a flexible way to characterize the repetition. 
We say that a (flat, two-dimensional) texture is cyclostationary if its windowed statistics are doubly periodic, meaning there exist two linearly independent vectors, $\boldsymbol{\tau}, \boldsymbol{\sigma} \in \mathbb{R}^2$ for which texture statistics $S(\mathbf{x})$ in a spatial neighborhood around location $\mathbf{x}$ satisfy $S(\mathbf{x} + \boldsymbol{\tau}) = S(\mathbf{x})$ and $S(\mathbf{x} + \boldsymbol{\sigma}) = S(\mathbf{x})$ for all $\mathbf{x} \in \mathbb{R}^2$. Note that the definition of statistics $S$ requires care. One generally wants them to encode all of the perceivable appearance information, and finding a single set of statistics $S$ that do this for all different types of textures has been the topic of decades of research, dating back to the Julesz conjecture~\cite{julesz,portilla}.

If we give ourselves permission to tailor the statistics $S$ to different types of textures (as we do here), then cyclostationarity describes many different texture types, including those composed of an isolated structural element (an ``interest point'') as in Row~1 of Fig.~\ref{fig:teaser}, where the intensity itself satisfies $I(\mathbf{x} + \boldsymbol{\tau}) = I(\mathbf{x})$ and $I(\mathbf{x} + \boldsymbol{\sigma}) = I(\mathbf{x})$ for orthogonal $\boldsymbol{\tau}$ and $\boldsymbol{\sigma}$ that are the height and width of the element. It similarly includes periodic textures (Row 2) with orthogonal $\boldsymbol{\tau}$ and $\boldsymbol{\sigma}$ being the spacing between tiles. With more general, higher-order statistics $S$, it even applies to a broad range of stochastic cyclostationary textures such as text (Row 3), where $\boldsymbol{\tau}$ and $\boldsymbol{\sigma}$ correspond to the vertical and horizontal spacing between rows and letters, respectively.

An important subtlety in the definition of cyclostationarity is the requirement that the texture has two linearly independent period vectors, $\boldsymbol{\tau}$ and $\boldsymbol{\sigma}$. This is necessary to facilitate shape estimation, since in the case where $\boldsymbol{\tau}$ and $\boldsymbol{\sigma}$ are linearly dependent (i.e., point in the same direction), shape estimation in the perpendicular direction cannot be obtained. One example of such an ambiguous texture is a set of parallel lines. (Imagine only the vertical lines in Row 2 of Fig.~\ref{fig:teaser}.)

\section{Warps} \label{sec:warps}

Our image patches are orthographic projections of an oriented texture plane, and the resulting mapping from the texture plane to the image plane is a two-dimensional transformation that has a special form~\cite{lobay2006shape}. We call this a \emph{warp}. 
\begin{definition}
A $2\times 2$ matrix $T$ is a \emph{warp} if
it can be written as $T = R_1 F R_2$ where $R_1$ and $R_2$ are rotation matrices and $F$ is a foreshortening matrix which is diagonal with values $1$ and $r \in (0, 1]$ along its diagonal.
\end{definition}
Such matrices have a singular value of $1$, and have positive determinant.
(We will not use the $r\leq 1$ property until discussing ambiguous cases in Section~\ref{sec:ambiguity}.)
Each warp has three degrees of freedom as described above, which distinguishes them from a general 2D translation-free affine transformation with four degrees of freedom. These warps are called
texture imaging transformations in~\cite{lobay2006shape}.

Warps have a useful property for the purposes of this paper:
\begin{lemma}
\label{lem:det}
If the $2\times 2$ matrix $T$ is a warp, then $\det(T^\top T - I) = 0$, where $I$ is the $2\times 2$ identity matrix.
\end{lemma}
\begin{proof}
If $T = R_1 F R_2$ then we can write $T^\top T = R_2^\top F^2 R_2$. We therefore have $T^\top T - I = R_2^\top F^2 R_2 - I = R_2^\top (F^2 - I)R_2$, so $\det(T^\top T - I) = \det(F^2 - I)$. Since $F^2 - I$ is a diagonal matrix with a $0$ entry along its main diagonal, we have $\det(T^\top T - I) = 0$.
\end{proof}

\section{Correct Geometry and Texture} \label{sec:geometry}

Our input is a set of image patches indexed by $i \in \{1, ..., N\}$. These are instances of patches from a flat cyclostationary texture process $\mathcal{T}$, with periods $\boldsymbol{\tau},\boldsymbol{\sigma}$ and statistics $S$, that have been spatially transformed by a set of warps $\{T_i\}_{i=1}^N$. We assume the existence of a texture correspondence algorithm that can identify and extract these image patches, and can register them by computing warps $\{W_i\}_{i=1}^N$ that explain all image patches $i \in \{1, ..., N\}$ by a single texture process $\mathcal{T}'$ that may be different from the generating one. Algorithms with this capability exist for various kinds of textures, including isolated texture elements~\cite{hartley}, compositions of SIFT keypoints~\cite{lobay2006shape}, and more general cyclostationary stochastic processes~\cite{verbin}.

Such an inferred texture process $\mathcal{T}'$ has its own $\boldsymbol{\tau}'$ and $\boldsymbol{\sigma}'$. Moreover, the computed warps $\{W_i\}_{i=1}^N$ of a successful algorithm must satisfy $W_i^{-1} T_i \boldsymbol{\tau} = \boldsymbol{\tau}'$ and $W_i^{-1} T_i \boldsymbol{\sigma} = \boldsymbol{\sigma}'$ for all $i \in \{1, ... N\}$. What remains is to characterize the relationship between the computed $\{W_i\}_{i=1}^N$ and the veridical $\{T_i\}_{i=1}^N$ (which also will give the relationship between $\mathcal{T}'$ and $\mathcal{T}$).

What we will show is that the computed warps are equal to the true warps (up to an inconsequential rotation) as long as there are $N\geq4$ input patches in general. We show this in two parts. First, cyclostationarity implies that all of the computed warps and generating warps are related by a single $2\times 2$ matrix $B$. Next, with the help of a small intermediate result, we show that four or more patches generically imply that the matrix $B$ is a rotation. 

To this end, we start with a definition.

\begin{definition} \label{def:goodwarps}
Let $\{T_i\}_{i=1}^N$ 
be a fixed set of warps.
We say that a set of warps $\{W_i\}_{i=1}^N$
is \emph{good} if 
there exist some $\boldsymbol{\tau}, \boldsymbol{\sigma} \in \mathbb{R}^2$
and some $\boldsymbol{\tau}', \boldsymbol{\sigma}' \in \mathbb{R}^2$ such that 
$\{W_i\}_{i=1}^N$ satisfies $W_i^{-1} T_i \boldsymbol{\tau} = \boldsymbol{\tau}'$ and $W_i^{-1} T_i \boldsymbol{\sigma} = \boldsymbol{\sigma}'$ for all $i \in \{1, ... N\}$.
\end{definition}

In the above definition, we do not actually care where the $\boldsymbol{\tau}, \boldsymbol{\sigma}$ and  $\boldsymbol{\tau}', \boldsymbol{\sigma}'$ come from, though in our setting these will be the periods of $\mathcal{T}$ and $\mathcal{T}'$ respectively.

\begin{lemma} \label{lemma:globalb}
Let $\{W_i\}_{i=1}^N$ be a good set of warps. 
Then $W_i = T_i B$ for all $i \in \{1, ..., N\}$,
where $B$ is some matrix with positive determinant.
\end{lemma}

\begin{proof}
By assumption, 
there exist $\boldsymbol{\tau}', \boldsymbol{\sigma}' \in \mathbb{R}^2$ such that for all $i \in \{1, ..., N\}$, $W_i^{-1} T_i \boldsymbol{\tau} =\boldsymbol{\tau}'$ and $W_i^{-1} T_i \boldsymbol{\sigma} = \boldsymbol{\sigma}'$. For all $i$, since $T_i$ and $W_i$ are invertible, we can define an invertible matrix $B_i \eqdef T_i^{-1}W_i$, and therefore $B_1^{-1} \boldsymbol{\tau} = ... = B_N^{-1} \boldsymbol{\tau} = \boldsymbol{\tau}'$ and $B_1^{-1} \boldsymbol{\sigma} = ... = B_N^{-1} \boldsymbol{\sigma} = \boldsymbol{\sigma}'$. We have that $\boldsymbol{\tau}$ and $\boldsymbol{\sigma}$ are two independent vectors which are in the null space of $B_1^{-1} - B_i^{-1}$ for all $i \in \{2, ..., N\}$, and therefore $B_1^{-1} - B_i^{-1} = 0$, and $B_1 = ... = B_N \eqdef B$.
The determinant sign follows since warps also have positive determinants.
\end{proof}

\begin{figure}
    \centering
    \includegraphics[width=0.45\textwidth]{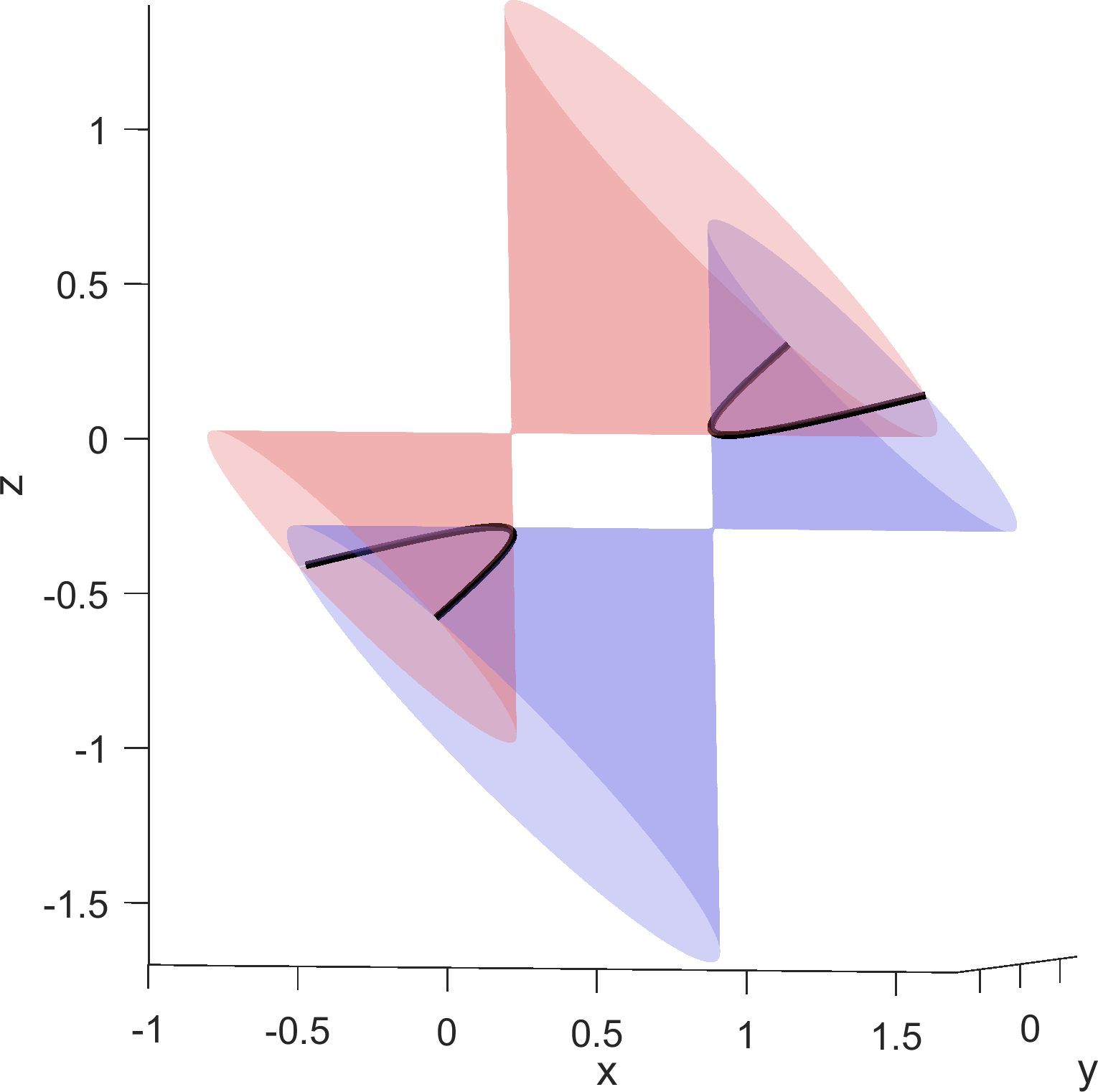}
    \caption{Each transformation $T_i$ maps to a point $(x_i,y_i,z_i)$ in this three-dimensional space, and in particular to a point on the red cone. The set of modified transformations $T_iB$ lies on a translated cone, with the translation determined by $B$. One example is shown in blue. A non-orthogonal $B$ exists only if a particular set of observed transformations $\{T_i\}_{i=1}^N$ corresponds to points contained in a planar slice of the original cone (e.g., on the black curve).}
    \label{fig:cone}
\end{figure}

Next we need the following simple lemma about cones:
\begin{lemma} \label{lemma:intersection}
The intersection of the two cones $y^2 = xz$ and $(y+b)^2 = (x+a)(z+c)$ for $a$, $b$ and $c$ which are not all zero, lies on a plane.
\end{lemma}
\begin{proof}
Subtracting the equation of the first cone from the equation of the second and rearranging, we obtain:
\begin{equation}
    cx - 2by + az = b^2 - ac,
\end{equation}
which is an equation of a plane.
\end{proof}

The next lemma is the heart of our argument.
\begin{lemma} \label{lemma:orthogonality}
Let $\{T_i\}_{i=1}^N$ be a set of warps, and assume $\{T_i B\}_{i=1}^N$ is also a set of warps for some invertible $2\times 2$ matrix $B$ with positive determinant. If $\{T_i^\top T_i\}_{i=1}^N$ contains four matrices that are affinely independent, then $B$ must be a rotation.
\end{lemma}

\begin{proof}
From Lemma~\ref{lem:det}
$\det(H_i) = 0$, where 
$H_i \eqdef T_i^\top T_i - I$.
Writing $H_i \eqdef \left[\begin{matrix}x_i & y_i \\ y_i & z_i \end{matrix}\right]$, we have that $x_iz_i - y_i^2 = 0$, i.e., $H_i$ is on the cone $y^2 = xz$ (see Fig.~\ref{fig:cone}).

Since $T_iB$ is also a warp, we also have that $\det(B^\top T_i^\top T_i B - I) = 0$. Because $B$ is invertible, we 
can write this as $\det(T_i^\top T_i - B^{-\top} B^{-1}) = 0$, or equivalently, $\det(H_i + I - B^{-\top} B^{-1}) = 0$. Therefore, $H_i$ is also on the cone $(y+b)^2 = (x+a)(z+c)$ where $I - B^{-\top} B^{-1} \eqdef \left[\begin{matrix} a & b \\ b & c \end{matrix}\right]$. Note that the two cones are identical if and only if $a = b = c = 0$, which is equivalent to $B$ being orthogonal.

By assumption, there exist at least four matrices $T_i^\top T_i$ which are not coplanar (in the three-dimensional space of symmetric $2\times 2$ matrices), and therefore there are at least four matrices $H_i$ which are not coplanar. But from Lemma~\ref{lemma:intersection}, the intersection of two translated cones must lie in a plane, unless the translation is zero. Therefore the two cones which contain the (non-coplanar) $H_i$ matrices must be identical, meaning that $a = b = c = 0$. This shows that $B$ is orthogonal. Its positive determinant makes $B$ a rotation.
\end{proof}

This says that all observable warps correspond to points on the red cone of Fig.~\ref{fig:cone}, and that ambiguity in the interpretation of geometry and texture can only occur when all observed transformations $\{T_i\}_{i=1}^N$ also correspond to points contained in some translated copy of that cone, such as the blue cone in the figure. This only happens when the observed transformations lie in some planar slice of the red cone, such as the black curve in  Fig.~\ref{fig:cone}.

Putting together Lemmas~\ref{lemma:globalb} and~\ref{lemma:orthogonality}, we arrive at the following
conclusion:
\begin{theorem}
\label{thm:main}
Let $\{T_i\}_{i=1}^N$ be a  set of warps and $\{W_i\}_{i=1}^N$ be a good set of warps. 
If $\{T_i^\top T_i\}_{i=1}^N$ contains four matrices that are affinely independent, then $W_i=T_i B$ for all $i \in \{1, ... N\}$, where $B$ is a rotation matrix.
\end{theorem}

Given a set of patches with true warps $\{T_i\}_{i=1}^N$, Theorem~\ref{thm:main} says that in the general case, the only possible good warps are the true warps $\{T_i\}_{i=1}^N$, up to an inconsequential rotation. This means that when $N\geq 4$ generic transformations are observed, and when a shape from texture algorithm such as~\cite{hartley, lobay2006shape, verbin} produces a good set of warps, the algorithm must return the true geometry.

\begin{figure}
    \centering
    \includegraphics[width=0.46\textwidth]{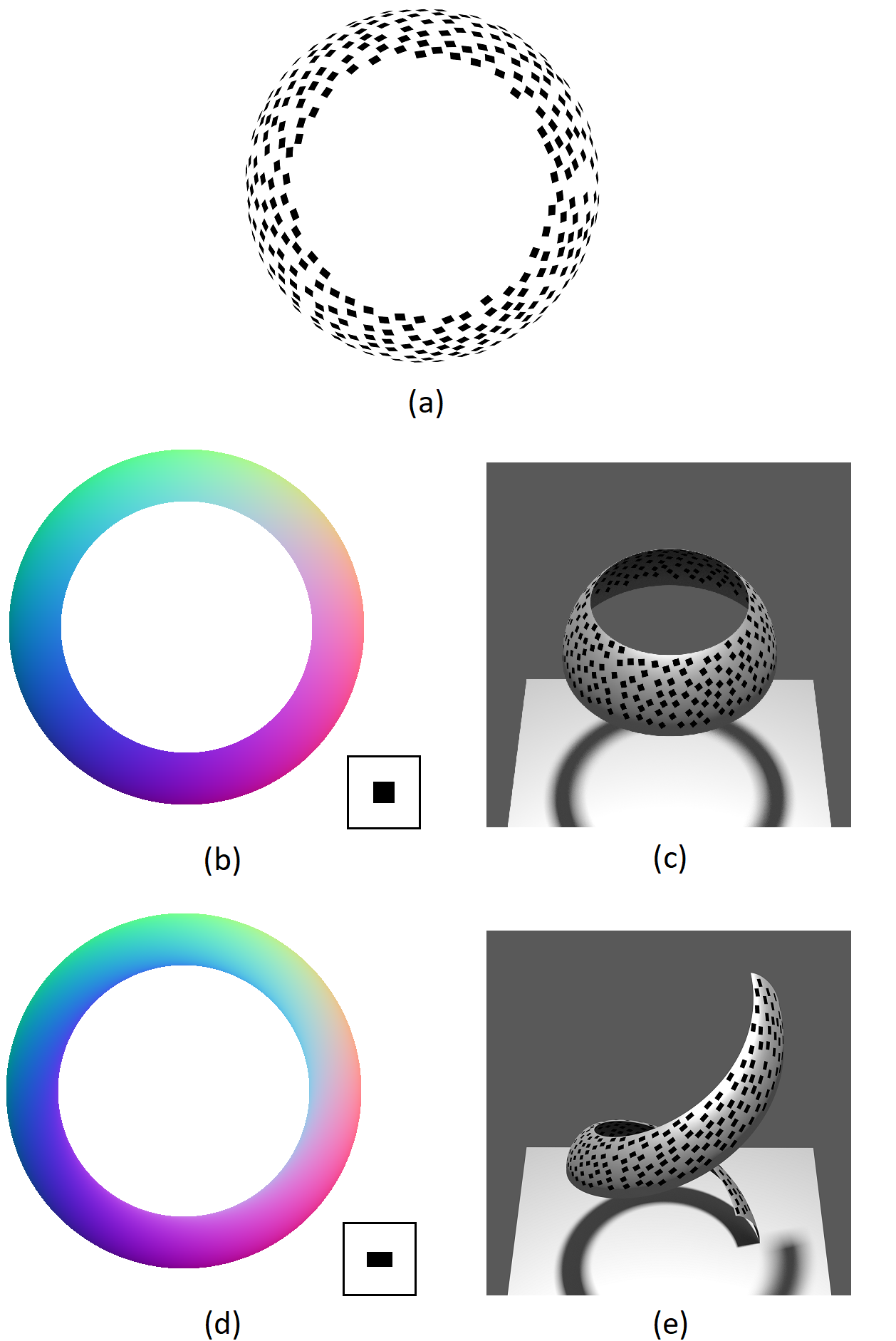}
    \caption{Example of ambiguous shape from texture, using $\lambda = 0.5$ in Eq.~\ref{eq:phi}. (a) Input image. (b-c) Correct interpretation: true surface normals and frontal texture element, and true surface from new viewpoint. (d-e) Alternative interpretation: frontal texture element transformed by $B^{-1}$, surface normals from warps transformed by $B$, and surface obtained by integrating the gradient defined by the alternative surface normals. The texture elements are scaled $\times 2$ for clarity.}
    \label{fig:sphere}
\end{figure}

\section{Ambiguous cases} \label{sec:ambiguity}

When all $\{H_i\}_{i=1}^N$ matrices lie in a planar slice of the cone, we do not guarantee that $B$ is restricted to being a rotation, because the conditions of Theorem~\ref{thm:main} will not hold. This is always the situation when $N=3$, since any three matrices are contained in a plane; but it can also occur for larger $N$ if the warps happen to correspond to matrices $\{H_i\}_{i=1}^N$ that are coplanar.

In fact, when all matrices $\{H_i\}_{i=1}^N$ lie in a plane, geometry can either be unique or ambiguous, depending on the plane. An ambiguous case is easy to illustrate by example.

\begin{example}
Define an annular section of hemispherical surface that has height function $z(x, y) = \sqrt{1 - x^2 - y^2}$ on domain $\lambda \leq x^2 + y^2 \leq 1$ with some choice of $\lambda \in (0, 1)$. Define a (conspiratorial) tangent orientation field on the surface,
\begin{equation}\label{eq:phi}
    \phi(x, y) = \frac{1}{2} \cos^{-1}\left(\frac{\lambda}{x^2 + y^2}\right),\quad \lambda \leq x^2+ y^2 \leq 1,
\end{equation}
and select the true warp at position $(x, y)$ to be:
\begin{equation}
    T(x, y) = \frac{1}{\sqrt{x^2+y^2}}\left[\begin{matrix}x & -y \\ y & x\end{matrix}\right] \left[\begin{matrix}z(x, y) & 0 \\ 0 & 1\end{matrix}\right] R(\phi(x, y)),
\end{equation}
where $R(\phi)$ is a matrix representing rotation by angle $\phi$.

Applying the warps $\{T(x, y)\}$ to a square texture element and using them to paint the hemispherical section produces the image shown in Fig.~\ref{fig:sphere}(a). Figure~\ref{fig:sphere}(b) shows the true surface normals (visualized as an RGB image by linearly scaling them to the range $[0, 1]$) and true texture element; and Figure~\ref{fig:sphere}(c) renders the shape from a different viewpoint and with shading.

Let us choose:
\begin{equation}
B = \left[\begin{matrix}\frac{1}{\sqrt{1-\lambda}} & 0 \\ 0 & \frac{1}{\sqrt{1+\lambda}}\end{matrix}\right].
\end{equation}
Then the set $\{W(x, y)\}$, where $W(x, y) = T(x, y) B$, is an alternative set of warps which, along with a rectangular element obtained by scaling the square using $B^{-1}$, also explains the observed image patches.

The alternative warps $\{W(x, y)\}$ imply an alternative field of normals, which is shown in Fig.~\ref{fig:sphere}(d) beside the scaled texture element. One can verify that the associated gradient field $\nabla z = \left(\frac{\partial z}{\partial x}, \frac{\partial z}{\partial y}\right)=-\frac{1}{n_z(x, y)}\left(n_x(x, y),n_y(x, y)\right)$ is irrotational, meaning it satisfies $\frac{\partial^2 z}{\partial x\partial y} = \frac{\partial^2 z}{\partial y\partial x}$. Irrotational gradients can be integrated over a simply connected domain, so we remove the positive $x$-axis $\{(x,y):x>0, y=0\}$ from $\{(x,y): \lambda \leq x^2 + y^2 \leq 1\}$ and integrate the gradient field over this simply connected region. The resulting surface is helical and is rendered in Fig.~\ref{fig:sphere}(e).

Both of the shapes in Figs.~\ref{fig:sphere}(c) \& (e) render exactly to (a) when viewed from above and without shading.  \qed

\end{example}

This example shows that when an ambiguity exists in shape from texture, the relation between the correct and alternative shapes is more complicated than, say, a simple convex-concave ambiguity. It also suggests that additional shape constraints, such as surface orientation at occluding contours or shading cues (as done in~\cite{white2006combining}), might be helpful in resolving these ambiguities.

A precise characterization of the planar slices that cause ambiguity is achievable but somewhat nuanced. When the matrices $\{H_i\}_{i=1}^N$ are all coplanar, they generically lie in a conic section that is an ellipse, a hyperbola, or a parabola. Because the eigenvalues of $T_i^\top T_i$ are both at most $1$, $H_i=T_i^\top T_i - I$ has non-positive eigenvalues, and $H_i$ must lie in the negative single cone $y^2 = xz$ where $z \leq 0$. Using Sylvester's law of inertia, if $B^\top T_i^\top T_i B - I$ is on the negative single cone, then so is $T_i^\top T_i - B^{-\top} B^{-1} = H_i + I - B^{-\top} B^{-1}$. Expanding on the proof of Lemma~\ref{lemma:orthogonality}, for a non-rotational $B$ to exist, the $\{H_i\}_{i=1}^N$ matrices must in fact be on an intersection of two shifted \emph{single} cones. One can show that such an intersection only exists when the conic section containing $\{H_i\}_{i=1}^N$ is a hyperbola. Therefore, if the conic section is an ellipse or a parabola, $B$ must be a rotation matrix, and geometry is unique. If the conic section is a hyperbola, one can show that a non-rotational $B$ matrix always exists such that all $\{T_iB\}_{i=1}^N$ are valid warps, and thus geometry is ambiguous as in the example above. 

To be complete, the remaining cases are when the conic section is degenerate and consists of a single point, single line, or two intersecting lines. One can easily show that in the cases of a single point or line there exists a set of non-rotational $B$ matrices (and geometry is ambiguous), while for two intersecting lines geometry is unique, since two intersecting lines cannot be obtained by intersecting two translated cones.

\section{Limitations and Practical Implications} \label{sec:limitations}

As noted in the introduction, our results apply to input patches that have already been cropped, and already aligned pairwise up to translation-free affine transformations. Techniques for harvesting and aligning patches in this way already exist for various texture types. For isolated texture elements like those in Fig.~\ref{fig:sphere} and the top row of Fig.~\ref{fig:teaser}, one can use the methods described in~\cite{matas,mikolajczyk}. For general textures---including stochastic textures like the bottom row of Fig.~\ref{fig:teaser}---one can use the method described in a companion paper to this one~\cite{verbin}, which employs a game-based approach to extract aligned patches while recovering smooth surface shape at the same time.

The condition of Theorem~\ref{thm:main}, which guarantees the existence of a unique explanation, relies on the input set of warps being \emph{good}. This is equivalent to having perfect cropping and alignment of the input patches. Thus, the condition of Theorem~\ref{thm:main} only applies to this case. It also only applies to orthographic projections, and does not provide any information about uniqueness in perspective or weak-perspective viewing conditions.

\end{document}